\pdfoutput=1

\documentclass[11pt]{article}

\usepackage[preprint]{acl}
\usepackage{enumitem}
\usepackage{amsthm}
\usepackage{tcolorbox}

\newtheorem{theorem}{Theorem}[section]

\newtheorem{corollary}{Corollary}[section]

\usepackage{times}
\usepackage{latexsym}

\usepackage[T1]{fontenc}

\usepackage[utf8]{inputenc}

\usepackage{microtype}
\usepackage{amsmath}
\usepackage{booktabs}
\usepackage{subfig}
\usepackage{array}    %
\usepackage{ragged2e}
\usepackage{multicol}
\usepackage{multirow}
\usepackage{inconsolata}
\usepackage{makecell}
\usepackage{graphicx}

\newcommand{\RomanOne}{\uppercase\expandafter{\romannumeral 1} }
\newcommand{\RomanTwo}{\uppercase\expandafter{\romannumeral 2} }
\newcommand{\RomanThree}{\uppercase\expandafter{\romannumeral 3} }
\newcommand{\RomanFour}{\uppercase\expandafter{\romannumeral 4} }

\definecolor{cvprblue}{rgb}{0.21,0.49,0.74}

\usepackage{algorithm}
\usepackage{algorithmicx}
\usepackage[noend]{algpseudocode}

\usepackage{amsmath,amsfonts,bm}

\def\eqref#1{equation~\ref{#1}}

\def\1{\bm{1}}

\DeclareMathAlphabet{\mathsfit}{\encodingdefault}{\sfdefault}{m}{sl}
\SetMathAlphabet{\mathsfit}{bold}{\encodingdefault}{\sfdefault}{bx}{n}

\title{Unifying Two Types of Scaling Laws from the Perspective of Conditional Kolmogorov Complexity}

\author{\bf Jun Wan\\
UBS AG\\
jun.wan@ubs.com
}

\usepackage[nameinlink,capitalize]{cleveref}

\begin{document}
\maketitle
{
\renewcommand{\thefootnote}{\fnsymbol{footnote}}
\footnotetext{Pre-print with preliminary results, work in progress.}
}
\begin{abstract}

In 2020, OpenAI proposed the first type of Scaling Laws, 
describing the relationships between model loss and the scale of parameters, 
data, and training computation. 
In 2024, OpenAI proposed the second type of Scaling Laws, 
describing the relationship between model inference performance and inference computation. 
In this paper, we analyze LLMs training and inference processes from the perspective 
of lossless compression using conditional Kolmogorov complexity, 
and unify these two types of Scaling Laws. We find that both types of 
Scaling Laws improve approximation of conditional Kolmogorov complexity 
by increasing execution steps of Turing machine. 
The first type of Scaling Laws increases execution steps by increasing number 
of model parameters. The second type of Scaling Laws increases execution steps by 
increasing the number of intermediate tokens.
    
\end{abstract}

\section{Introduction}
\label{sec:intro}

In October 2022, OpenAI released ChatGPT 3.5 to the public, demonstrating the success of the first type of Scaling Laws proposed in 2020 \cite{1}. 
In September 2024, OpenAI released o1-preview and introduced test-time compute Scaling Laws \cite{2}. 
Subsequently, the release of DeepSeek-R1 \cite{27} by DeepSeek further validated the effectiveness of test-time compute Scaling Laws.
The success of these models highlights the significant impact of both types of Scaling Laws.

\citet{3} has long advocated that "a model that compresses well generalizes well". 
\citet{4} views Large Language Models (LLMs) as powerful lossless compressors. 
In this paper, we also analyze LLMs from the perspective of lossless compression. 
However, we differ by adopting the approach from NNCP \cite{5}, viewing the training process of LLMs as a form of lossless compression, where the model is trained solely on the data stream to be compressed.

In the realm of theoretical research, the analysis of lossless compression based on Kolmogorov complexity is a prevalent technical approach \cite{24}. Employing conditional Kolmogorov complexity as a theoretical tool, we systematically investigate the training and inference processes of LLMs.

\begin{figure}[ht]
    \centering
    \includegraphics[width=0.45\textwidth]{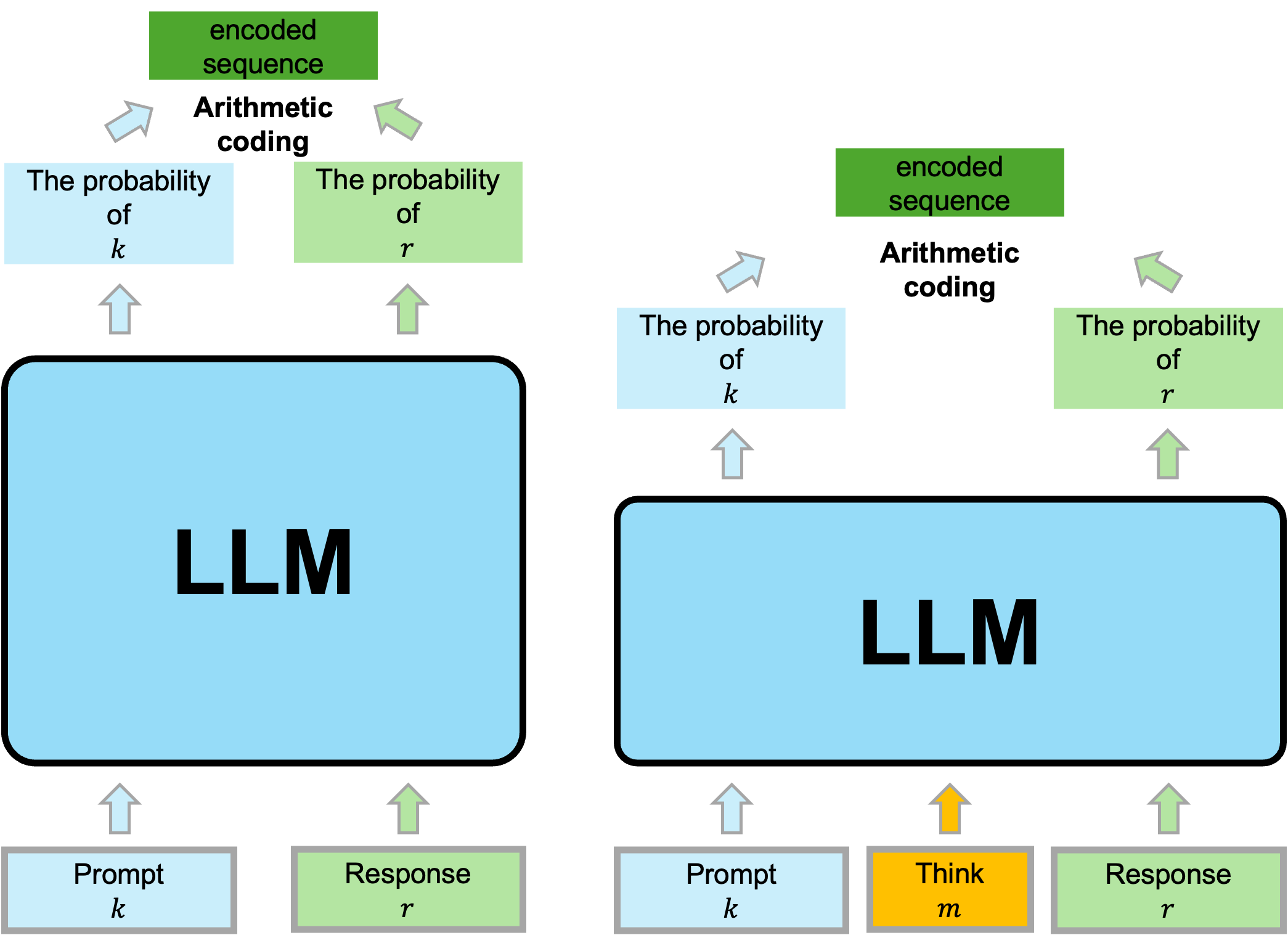}
    \caption{The left figure achieves more efficient compression of the string $(k, r)$ by increasing the model parameters, while the right figure enhances the compression efficiency of the string $(k, r)$ by introducing more intermediate tokens.}
    \label{fig:picture1}
\end{figure}
\noindent\textbf{This Work and Contribution}:
\begin{enumerate}[topsep=2pt, itemsep=0pt, leftmargin=12pt]
    \item We provide a detailed explanation of the relationship between lossless 
    compression and LLMs training, reviewing how to train on the stream of data that 
    is to be compressed.
    \item We model the compression process using conditional Kolmogorov complexity, 
    demonstrating that training LLMs approximates the upper bound of joint Kolmogorov complexity. 
    Optimizing joint Kolmogorov complexity inherently provides regularization.
    \item Although Kolmogorov complexity is not computable, we prove that theoretically 
    there exist decoder-only transformer such that $\displaystyle \lim_{t \to \infty} M(t, x ,y) = C(x\mid y)$.
    \item Through conditional Kolmogorov complexity, we analyze the theoretical limits 
    of LLMs in the inference process, showing that there exist infinite problems that 
    cannot be "solved" by models.
    \item We unify two types of Scaling Laws. Both types of Scaling Laws improve 
    approximation of conditional Kolmogorov complexity by increasing execution 
    steps of Turing machine.
\end{enumerate}

\section{Related Work}
\label{sec:related_work}
\noindent\textbf{Lossless Compression.}
Fabrice Bellard open-sourced NNCP \cite{5} in 2019, a tool for lossless data compression using neural networks, 
and upgraded it in 2021 to use autoregressive models \cite{6} for lossless compression. 
NNCP trains models on the data stream to be compressed. Its idea can be traced back to \cite{7}. 
They explained how to achieve lossless compression using machine learning model predictions and arithmetic coding. 
In fact, lossless compression is equivalent to prediction. 
\citet{4} further advocates viewing LLMs as powerful universal lossless compressors, analyzing the first type of Scaling Laws and the role of tokenization from a lossless compression perspective.

\noindent\textbf{Kolmogorov Complexity.}
Kolmogorov complexity is closely related to many fundamental principles in machine learning. 
For example, Minimum Description Length (MDL) proposed by \cite{8} can be viewed as a computable approximation of Kolmogorov complexity, while the Maximum Likelihood Principle \cite{9} and Maximum Entropy Principle \cite{10} can be seen as special cases of the MDL principle. 
Additionally, \citet{11} further explains the relationships between learning theories, such as PAC learning theory and Kolmogorov complexity. 
\citet{12} studied the connection between the Grokking phenomenon in neural networks and MDL through Kolmogorov complexity.

\section{Background}
\label{sec:background}

\subsection{Turing Machines, Neural Networks \& LLMs}

The Turing machine \citet{28}, proposed by British mathematician Alan Turing, 
is an abstract computational model designed to explore the computability of problems. 
As a cornerstone of computer science theory, the Turing machine has the capability 
to simulate the execution of any algorithm, providing an important framework for 
understanding the nature of computation.

From a theoretical perspective, Turing machines are enumerable. 
Among all Turing machines, there exists a special class called universal Turing machines, 
denoted as $U$, which can simulate the behavior of any other Turing machine.

Neural networks operating on modern electronic computers with finite precision 
are equivalent to a Turing machine that halts in finite steps, i.e., 
a total recursive function. The process of training neural networks 
through training data is essentially a process of searching for a specific Turing machine.

Next, let's briefly review LLMs. LLMs are based on neural network architectures and are 
trained through "next token prediction" tasks. Given an input sequence (training data) 
$x_{0:t} = [s_0,x_1,x_2,\cdots,x_t]$, the model's training objective 
is to predict the next token $x_{t+1}$ in the sequence \cite{13}. Here, $s_0$ represents 
the start token, which is a special token used to indicate the beginning of the sequence.
This process can be formalized as the following optimization problem:
\begin{equation} \label{eq:1}
\min_{\theta} -\frac{1}{n}\sum_{t=1}^{n} \log P(x_{t}\mid x_{0:t-1}; \theta)
\end{equation}
where $P(x_{t}\mid x_{0:t-1}; \theta)$ represents the conditional 
probability of predicting $x_{t}$ given the context sequence $x_{0:t-1}$ 
under the neural network model with parameters $\theta$, and $n$ represents the sequence length.

From a computational theory perspective, LLMs can also be viewed as a special type of Turing machines that minimize \Cref{eq:1}. The core function of this Turing machine is to calculate the conditional probability distribution of the next token in a predefined vocabulary ${\mathcal{S}}$ based on the given context sequence through complex probability calculations.

\subsection{Dynamic Arithmetic Coding}

Dynamic arithmetic coding \cite{14,15} is an adaptive data compression algorithm 
that achieves efficient coding by updating symbol probability distributions 
in real-time.

Let the symbol set be $\mathcal{S} = \{s_1,s_2,\cdots,s_m\}$, 
where each symbol $s_i$ is a sample of a discrete random variable $S$. 
At time $t$, the probability of symbol $s_i$ is denoted by $P_t(s_i),t=1,2,\cdots,n$, 
and the probabilities satisfy $\displaystyle \sum_{i=1}^m P_t(s_i)=1$.

\noindent\textit{\textbf{Arithmetic Encoding Process:}}
\begin{enumerate}[topsep=1pt, partopsep=0pt,itemsep=0pt]
    \item Initialize: interval $[l_0,h_0) = [0,1)$
    \item Update interval: For input symbol $s_{i_t}$ at time $t$, update interval to:
    \begin{align*}
        l_t = l_{t-1}+\sum_{j=1}^{i_t-1}P_t(s_j)\cdot (h_{t-1}-l_{t-1})\\
        h_t = l_{t-1}+\sum_{j=1}^{i_t}P_t(s_j)\cdot (h_{t-1}-l_{t-1})
    \end{align*}
    \item Probability update: Update $P_{t+1}(s_i),i=1,2,\cdots,m$
    \item Repeat: Continue processing next symbol in the sequence until all symbols are processed
    \item Output: Any number in the final interval $[l_n, h_n]$ as the encoding result. In binary, to achieve shortest encoding, select the decimal with shortest binary representation in the interval.
\end{enumerate}

\noindent\textit{\textbf{Arithmetic Decoding Process:}}

The decoder reconstructs the original symbol sequence through reverse operations, 
using the same probability distributions step by step. 
Note that, during the decoding process, apart from the 
arithmetic encoding information, we also need to know the number of decoding iterations.
You can find an example in \Cref{sec:example_of_arithmetic}.

It's important to emphasize that predicted probability 
distributions $\hat{P}_t(s_i)$ can be used for arithmetic coding of symbol sequence. 
The closeness between $\hat{P}_t$ and the true probability distribution $P_t$ directly 
affects coding efficiency - the closer they are, the shorter the average coding length. 

\subsection{LLMs \& Lossless Compression}\label{sec:llmsandlossless}

In this section, we'll explore how to achieve lossless data compression using LLMs.

Suppose we have an original message denoted as $x$, where $x=[s_0,x_1,x_2,x_3,\cdots,x_n]$, with $s_0$ being a special character indicating the start of the original message. Each $x_i$ represents a token in the original message, coming from a finite symbol set $\mathcal{S}$.

Given a LLMs $f$, let $\theta_t$ represent the model parameters at time $t$. Based on this, the model's probability distribution prediction for the $(t+1)$-th token can be formally expressed as:
\begin{equation}
\hat{p}_{t+1} = f(x_{t+1}\vert x_{0:t} ;\theta_t) 
\end{equation}
where $\hat{p}_{t+1}$ represents the model's probability distribution prediction for the $(t+1)$-th token, which is a discrete distribution.

The LLMs $f$'s code (such as the Pytorch training code) and various random seeds (such as parameter initializations, etc.) are packaged into a single piece of information. We denote this information as $F$. Now we begin lossless compression of the data $x$. Assume there is an information receiver and an information sender.

The information sender initializes neural network parameters $\theta_0$ according to $F$.

\begin{enumerate}[topsep=1pt, partopsep=0pt,itemsep=0pt]
    \item At time 0: 
    \begin{enumerate}
        \item Execute $f(x_1 \mid s_0 ;\theta_0)$. Output the prediction probability $\hat{p}_1$ for the first token.
        \item Use arithmetic coding on character $x_1$ with $\hat{p}_1$. Note that while $x_1$'s true probability distribution $p_1$ is unknown, this does not affect our ability to perform arithmetic coding on $x_1$ using the predicted probability distribution $\hat{p}_1$. Note that at this point, arithmetic coding has just begun its first interval division, then selected an interval based on $x_1$. We denote this interval as $[l_1,r_1]$.
        \item Perform backpropagation based on the true $x_1$ and predicted $\hat{p}_1$ to update $f$'s parameters, obtaining $\theta_1$.
    \end{enumerate}    
    \item At time 1:  
    \begin{enumerate}
        \item Execute $f(x_2 \mid s_0,x_1;\theta_1)$. Output the prediction probability $\hat{p}_2$ for the second token.
        \item Use arithmetic coding on character $x_2$ with $\hat{p}_2$. Note that at this point, arithmetic coding performs a second interval division based on $[l_1,r_1]$, then selects an interval based on $x_2$. We denote this interval as $[l_2,r_2]$. It's easy to see that as the arithmetic coding processes, intervals are continuously subdivided, and the probability $\hat{p}_t$ used for each division changes.
        \item Perform backpropagation based on the true $x_2$ and predicted $\hat{p}_2$ to update $f$'s parameters, obtaining $\theta_2$.
    \end{enumerate}
\end{enumerate}

\noindent We continuously repeat steps 2.a through 2.c until $n-1$. Finally, we will obtain an interval $[l_{n},r_{n}]$ and parameters $\theta_{n}$. We select the shortest binary decimal $z_n$ within $[l_{n},r_{n}]$ as the arithmetic coding for the entire information $x$.

This completes the overall lossless compression process. The compressed information consists of three parts:

\begin{itemize}[topsep=2pt, itemsep=0pt, leftmargin=12pt]
    \item Arithmetic coding $z_n$.
    \item Program information $F$.
    \item Required number of decoding iterations $d$.
\end{itemize}

Through the above process, we'll discover that the compressed information does not include the model parameters. In fact, we don't need to transmit the model parameters to perform lossless decompression.

The information sender transmits the three compressed components to the information receiver via bitstream. The information receiver now begins decoding.

The information receiver executes the code according to $F$ and initializes neural network parameters $\theta_0$. Note that since $F$ contains all random seed information, this $\theta_0$ is identical to the $\theta_0$ from the earlier steps.

\begin{enumerate}[topsep=1pt, partopsep=0pt,itemsep=0pt]
    \item At time 0:
    \begin{enumerate}
        \item Execute $f(x_1 \mid s_0 ;\theta_0)$. Output the prediction probability $\hat{p}_1$ for the first token.
        \item Perform arithmetic decoding based on $\hat{p}_1$ and the received $z_n$. The decoding process makes the first interval division based on $\hat{p}_1$. The information receiver will find that $z_n$ lies within the interval $[l_1,r_1]$. Thus, the first token $x_1$ is decoded.
        \item Perform backpropagation based on the decoded $x_1$ and predicted $\hat{p}_1$ to update $f$'s parameters, obtaining $\theta_1$. Note that since all random seeds are identical between the information sender and receiver, this $\theta_1$ is identical to the previous $\theta_1$.
    \end{enumerate}
    \item At time 1:
    \begin{enumerate}
        \item Execute $f(x_2 \vert s_0,x_1 ;\theta_1)$. Output the prediction probability $\hat{p}_2$ for the second token.
        \item Perform arithmetic decoding based on $\hat{p}_2$ and the received $z_n$, decoding the second token $x_2$. The decoding process makes a second interval division of $[l_1,r_1]$ based on $\hat{p}_2$. The information receiver will find that $z_n$ lies within the interval $[l_2,r_2]$. Thus, the second token $x_2$ is decoded.
        \item Perform backpropagation based on the decoded $x_2$ and predicted $\hat{p}_2$ to update $f$'s parameters, obtaining $\theta_2$.
    \end{enumerate}
\end{enumerate}

\noindent We continuously repeat steps 2.a through 2.c until $n-1$ (where $n-1$ is obtained from reading information $d$). Finally, we will have losslessly decoded $z_n$ back into the original message $x$. In the entire decoding process, the transmission of the model parameters $\theta_n$ is not required.

If we use probability distribution $\hat{p}$ to perform arithmetic coding 
on $S$, we can calculate that its average coding length satisfies the following relationship:
\begin{equation}\label{eq:3}
    \begin{aligned}
        \bar{L} &\leq -\sum_{s\in {\mathcal{S}}}p(s)\log_2 \hat{p}(s) +2 \\
        &:= H(p,\hat{p})+2
    \end{aligned}
\end{equation}
where $\bar{L}$ denotes the average coding length, $p$ is the true distribution of $S$, 
and $H(p,\hat{p})$ is the entropy between distributions $p$ and $\hat{p}$. 
We will use \Cref{eq:3} to calculate the lossless compression efficiency of LLMs.

The challenge of calculating compression rates lies in estimating the coding length $\vert z_n \vert$, where $\vert \cdot \vert$  denotes the length.

For dataset $x$, let $p_t$ represent its true distribution in autoregressive modeling:
\begin{equation*}
    p_{t} = p(x_t \mid x_{0:t-1})
\end{equation*}
The average coding length of arithmetic coding $z_n$ satisfies the following relationship:
\begin{equation}
    \begin{aligned}  
        \bar{L}_{z_n} &\leq \frac{1}{n}(2n+\sum_{t=1}^nH(p_t,\hat{p}_t))\\ &= 2+ \frac{1}{n}\sum_{t=1}^nH(p_t,\hat{p}_t) \\ &= 2-\frac{1}{n}\sum_{t=1}^n\log P(x_t \mid x_{0:t-1};\theta_{t-1})
    \end{aligned} 
\end{equation}
The $\vert F\vert$ and $\vert d \vert$ are basically fixed and negligible compared to $\vert z_n \vert$. To maximize compression, we need to minimize the average coding length of $z_n$. Since directly reducing $z_n$ is difficult, and given that we have an upper bound for $z_n$, we can instead try to reduce $z_n$'s upper bound, namely:
\begin{equation} \label{eq:5}
    \min -\frac{1}{n}\sum_{t=1}^n\log P(x_t \mid x_{0:t-1};\theta_{t-1})
\end{equation}
Obviously, the optimization objective described in \Cref{eq:5} exactly matches the training objective of LLMs. This indicates that the training process of LLMs can essentially be understood as a continuous compression process of the training data.

\subsection{Kolmogorov complexity}

Kolmogorov complexity \cite{11} is a concept for measuring the information content of an object (such as strings, numbers, etc.). 
Specifically, it represents the length of the shortest program needed to generate that object on a given universal Turing machine $U$.

Given a universal Turing machine $U$, the Kolmogorov complexity $C_U(x)$ of string $x$ is defined as:
\begin{equation}
    C_U(x) = \min \{l(p): U(p)=x\}
\end{equation}
where:
\begin{itemize}[topsep=2pt, itemsep=0pt, leftmargin=12pt]
    \item $p$ is a program, and $l(p)$ is the length of the program.
    \item $U(p)=x$ means program $p$ outputs $x$ when run on universal Turing machine $U$.
\end{itemize}

Kolmogorov complexity measures the length of the shortest description needed to generate $x$. If $x$ has regularity, it can be described by a relatively short program, resulting in low Kolmogorov complexity; if $x$ is random, it requires a longer program to describe, resulting in high Kolmogorov complexity.

We emphasize two important properties here for Kolmogorov complexity:
\begin{enumerate}[topsep=0pt, partopsep=0pt,itemsep=0pt]
    \item \textbf{Invariance Theorem:} While Kolmogorov complexity depends on the choice of universal Turing machine $U$, according to the invariance theorem, the Kolmogorov complexity under different universal Turing machines only differs by a constant. Therefore, we often omit the subscript $U$ in the definition of Kolmogorov complexity.
    \item \textbf{Uncomputability:} Kolmogorov complexity is uncomputable, meaning there is no algorithm that can precisely calculate the Kolmogorov complexity of any arbitrary string.
\end{enumerate}

Conditional Kolmogorov complexity is an extension of Kolmogorov complexity that measures the shortest description length needed to generate an object given additional information (conditions). Specifically, it represents the length of the shortest program needed to generate object $x$ on universal Turing machine $U$, where the program can utilize additional information $y$.

Given a universal Turing machine $U$, the conditional Kolmogorov complexity $C_U(x \mid y)$ of string $x$ given condition $y$ is defined as:
\begin{equation}
    C_U(x \mid y) = \min \{l(p): U(p,y)=x\} 
\end{equation}
where:
\begin{itemize}[topsep=2pt, itemsep=0pt, leftmargin=12pt]
    \item $p$ is a program, and $l(p)$ is the length of the program.
    \item $U(p,y)=x$ means program $p$ runs on universal Turing machine $U$ with input $y$ and it outputs $x$.
\end{itemize}

Conditional Kolmogorov complexity measures the shortest description length needed to generate $x$ given knowledge of $y$. If $y$ provides useful information about $x$, then the program to generate $x$ might be shorter, thus reducing the conditional Kolmogorov complexity.

Joint Kolmogorov complexity is another extension of Kolmogorov complexity, used to measure the shortest length needed to describe two objects $x$ and $y$ together. It represents the length of the shortest program needed to generate both $x$ and $y$ on universal Turing machine $U$.

Given a universal Turing machine $U$, the joint Kolmogorov complexity $C_U(x,y)$ of strings $x$ and $y$ is defined as:
\begin{equation}
    C_U(x,y) = min \{l(p): U(p) = \langle x,y\rangle\}
\end{equation}
where:
\begin{itemize}[topsep=2pt, itemsep=0pt, leftmargin=12pt]
    \item $p$ is a program, and $l(p)$ is the length of the program.
    \item $U(p) = \langle x,y\rangle$ means program $p$ outputs some encoding of $x$ and $y$ when run on universal Turing machine $U$ (e.g concatenating $x$ and $y$ into a single string).
\end{itemize}

Joint Kolmogorov complexity measures the shortest description length needed to generate both $x$ and $y$ simultaneously. If there exists some correlation or regularity between $x$ and $y$, then generating their joint description might be shorter than generating their descriptions separately.

\subsection{Two Types of Scaling Laws}

The Scaling Laws in LLMs can primarily be categorized into two types:

\begin{itemize}[topsep=2pt, itemsep=0pt, leftmargin=12pt]
    \item Pre-training Scaling Laws
    \item Inference Scaling Laws
\end{itemize}

\citet{1} systematically studied the impact of model parameter scale on language model performance and proposed corresponding Scaling Laws. It focuses on resource optimization during the pre-training process, in other words, to improve model performance through increasing key resources such as data volume, model parameter count, and computational power. The proposal of such Scaling Laws also laid an important foundation for training and optimization of LLMs.

\citet{16} conducted an in-depth study on the feasibility of improving LLMs performance by increasing computational resources during the inference process. This research direction was empirically supported by OpenAI's o1 model released in September 2024. Research shows that significantly increasing computational resources and time investment during the inference process can effectively improve model performance in complex tasks such as mathematics, programming, and logical reasoning. This finding not only validates the importance of computational scaling during inference but also provides a new technical pathway for optimizing LLMs performance. The outstanding performance of OpenAI's o1 model \cite{17} further confirms that there is a significant positive correlation between inference process resource investment and model performance, particularly demonstrating important practical value when handling high-complexity tasks.
\section{Analysis of LLMs Pre-training and Inference from the Perspective of Conditional Kolmogorov Complexity}

In this section, we will analyze the pre-training and inference processes of LLMs using conditional Kolmogorov complexity, and reach the following conclusions:

\begin{enumerate}[topsep=2pt, itemsep=0pt, leftmargin=12pt]
    \item The pre-training process of LLMs is equivalent to approximating joint Kolmogorov complexity $C(x,y)$.
    \item Although Kolmogorov complexity is uncomputable, we prove that theoretically there exist total recursive functions or decoder-only transformer such that $\displaystyle \lim_{t \to \infty} M(t,x,y) = C(x \mid y)$.
    \item LLMs are constrained by the number of Turing machine execution steps, which prevents them from solving certain problems.
    \item Theoretically, we have unified the two types of Scaling Laws. Both types of Scaling Laws work by increasing the number of Turing machine execution steps to better approximate conditional Kolmogorov complexity.
\end{enumerate}

\subsection{The Relationship Between LLMs Pre-training and Conditional Kolmogorov Complexity}

Let us denote training data as $x$ and the model as $y$. Using data $x$ to train a model $y$ with good "generalization" ability can be viewed as searching for a $y$ that minimizes their joint Kolmogorov complexity $C(x, y)$.
\begin{equation}\label{eq:9}
    \min_y C(x,y)
\end{equation}
Next, we'll demonstrate the reasonableness of this formula. According to \citeposs{18} work, $C(x, y)$ can be further decomposed into the following form:
\begin{align} \label{eq:10}
    C(x,y) = & \ C(y) \nonumber + C(x \mid y) \nonumber \\
            & + O(\log(C(x,y)))
\end{align}
$C(x \mid y)$ represents the shortest description length of $x$ given model $y$. $C(y)$ represents the shortest description length of model $y$. In other words, finding a $y$ that minimizes $C(x, y)$ is equivalent to finding the simplest possible model $y$ that minimizes $C(x \mid y)$.

In fact, the right side of \Cref{eq:10} can be viewed as the MDL in machine learning. Here, $C(y)$ represents model complexity, while $C(x \mid y)$ represents the encoding length of data under that model. $C(y)$ can be viewed as model regularization, while $C(x \mid y)$ reflects how well the model fits the data.

It's important to note that we measure model complexity using Kolmogorov complexity rather than simply counting the number of model parameters. This approach is quite reasonable. For example, for a neural network with 10,000 parameters, when all parameters are 1, the network's complexity is very low; when the parameters are completely random, the network's complexity increases significantly.

If we solely optimize $y$ to minimize $C(x \mid y)$, it may result in overfitting. For instance, if we directly set $y = x$, then model $y$ becomes exactly equivalent to data $x$, which is essentially memorization. In this case, $C(x \mid y)$ would be $O(1)$, while $C(y)$ would be very large. A neural network with good generalization capabilities should start from completely random parameter initialization and gradually learn, during which process both $C(y)$ and $C(x \mid y)$ should decrease gradually.

Below, we will demonstrate through rigorous mathematical derivation that LLMs pre-training is equivalent to directly approximating joint Kolmogorov complexity $C(x, y)$ upper bound, and naturally considers model complexity during the training process.

\begin{theorem}
    Given a universal Turing machine $U$, the joint Kolmogorov complexity $C(x, y)$ of strings $x$ and $y$ satisfies the following inequality:
    \begin{align}\label{eq:11} 
    C(x,y) \leq & C(y) + C(x \mid y)\nonumber \\ 
          & +2l(C(y)) +O(1)
    \end{align} 
    where $l$ represents string length, and $O(1)$ is a constant related to Turing machine $U$.
\end{theorem}

The proof is in \Cref{subsec:proof_of_theorem_41}.

\begin{corollary}
    The pre-training process of LLMs is actually a computable (total recursive) approximation of the right side of \Cref{eq:11}.
\end{corollary}

We use the content from \Cref{sec:llmsandlossless} to perform lossless compression of $x$. Note that during the lossless compression process, we not only obtain the encoding of $x$ based on $y$ but also obtain the parameters of model $y$.

The compressed information consists of three parts:

\begin{itemize}[topsep=2pt, itemsep=0pt, leftmargin=12pt]
    \item Arithmetic coding $z_n$.
    \item Program information $F$.
    \item Required number of decoding iterations $d$.
\end{itemize}

Obviously, $C(y) + C(x \mid y) + O(1) \leq \vert z_n \vert + \vert F \vert+\vert d \vert$ and $2l(C(y)) + O(1) \leq 2l(\vert z_n \vert + \vert F \vert+\vert d \vert)$. The smaller the encoding length of $z_n$, the closer it gets to the upper bound of $C(x, y)$. In other words, the pre-training process of LLMs is essentially searching for a model $y$ of moderate complexity such that, given $y$, the encoding length of input data $x$ is minimized. This objective reflects the core idea of the MDL principle, which seeks optimal balance between model complexity and data fitting.

\subsection{The Relationship Between LLMs Inference and Conditional Kolmogorov Complexity}

We rewrote \Cref{eq:9} and applied it to infinite $x$, obtaining the following expression:
\begin{equation}\label{eq:12}
    y^* = \arg \min_y \sum_{x} C(x,y)
\end{equation}
In the LLMs, $x$ here represents a string concatenated from $(k, r)$, where $k$ is our input (prompt) and $r$ is our expected output (response). We \emph{conjecture} that $C(x \mid y^*)$ represents the ideal state of artificial general intelligence (AGI), as it can perform lossless compression on any $x$ we care about, indicating it has predictive capability for any $x$. Note that if $x$ cannot be losslessly compressed, it means $x$ is "random" and lacks regularity, therefore such $x$ are not our targets of interest.

\emph{Assuming} $y^*$ in \Cref{eq:12} exists and is known, for simplicity of expression, we will denote $y^*$ as $y$ in the following sections.

Next, we focus on whether we can represent $C(x \mid y)$ through neural networks. Obviously, the answer is negative because $C(x \mid y)$ is uncomputable. However, fortunately, we can approximate $C(x \mid y)$ through total recursive functions.

\begin{theorem}\label{the:1}
    For each $x$, there exists a total recursive function $\phi(t,x,y)$ that is monotonically decreasing with $t$, such that $\displaystyle \lim_{t \to \infty} \phi(t,x,y)= C(x\mid y)$.
\end{theorem}

The proof is in \Cref{subsec:proof_of_theorem_42}.

\begin{corollary}\label{cor:2}
    Theoretically, there exists a series of decoder-only transformer $M(t,x,y)$, such that $\displaystyle \lim_{t \to \infty} M(t,x,y) = C(x \mid y)$.
\end{corollary}

\citet{19} proved that decoder-only transformer are Turing complete under infinite-precision 
rational numbers and hardmax attention mechanisms. Under this assumption, we can construct 
decoder-only transformer $M(t, x,y)$ to simulate $\phi(t,x,y)$. Note that we currently 
cannot prove whether the LLMs and arithmetic coding obtained using the content 
from \Cref{sec:llmsandlossless} can infinitely approximate $C(x \mid y)$.

\begin{corollary}\label{cor:3}
    Given $x$ and $t$, we cannot determine whether there exists $M(t, x, y) = C(x \mid y)$.
\end{corollary}

If we could determine this, we would conclude that $C(x \mid y)$ is computable.

\begin{corollary}\label{cor:4}
    For any $\epsilon > 0$ and each $t$, there exist infinite $x$ such that 
    there is always an $\epsilon$ error between $M(t, x, y)$ and $C(x \mid y)$.
\end{corollary}

\Cref{cor:2} tells us that theoretically, we can construct a series of decoder-only transformer 
to gradually approximate $C(x \mid y)$ (i.e., AGI). 
However, \Cref{cor:3} and \Cref{cor:4} indicate that in practice, 
we cannot precisely determine how close we are to $C(x \mid y)$. 
Particularly, \Cref{cor:4} further shows that there will always be some $x$ that 
have a certain error with $C(x \mid y)$. 
Unfortunately, we currently cannot determine whether these $x$ with errors are indeed $x$ that 
we care about (since some $x$ might be completely meaningless random strings). 
However, When we use practically applied decoder-only transformer
$m(x)$ to approximate $C(x \mid y)$, the following interesting corollaries arise.

\begin{corollary}\label{cor:5}
    For a practically applied decoder-only transformer $m(x)$, there exist meaningful $x$ such that $m(x) \neq C(x \mid y)$.
\end{corollary}

Note that in the following proof, we assume $NP \neq P$. In fact, even without this assumption, 
as long as we can construct a problem that can not be solved in polynomial time, 
the conclusion still holds.

Suppose string $(k, r) = x$ is as follows:
\begin{tcolorbox}[arc=10pt]
\texttt{\textbf{$k$}: Given $n$ sets \{n sets\}, for each set, does there exist a subset whose sum equals exactly 42?}

\texttt{\textbf{$r$}: yes,no,no,yes,no......}
\end{tcolorbox}
\noindent where \{n sets\} are $n$ integer sets, and $c$ is a string of length $n$ composed 
of "yes" or "no", determined by the content of \{n sets\}. Obviously, $x$ represents 
the Subset Sum Problem, which is a NP-complete problem. Since we assume $NP \neq P$, 
this means there exists no polynomial-time Turing machine that can solve this problem.

For a decoder-only transformer, $m(x)$ is a Turing machine with polynomial-time complexity. 
Assuming for any $x$ (here referring to the use of different sets of integers \{n sets\}), if 
$m(x)=C(x\mid y)$, then we have found a Turing machine $m(x)$ with polynomial-time complexity that 
solves the Subset Sum Problem, leading to a contradiction. Therefore, there exist some $x$ for which 
$m(x)\neq C(x\mid y)$.

As shown in \Cref{the:1}, in order to better approximate the conditional Kolmogorov complexity 
$C(x\mid y)$, we need to find Turing machines that execute more steps. In decoder-only transformer, 
there are two ways to increase execution steps, they are illustrated in \Cref{fig:picture1}:

\begin{itemize}[topsep=2pt, itemsep=0pt, leftmargin=12pt]
    \item Increase the model's parameter count: As the model's parameter count increases, under the same $l(x)$, execution steps becomes larger.
    \item Increase the encoding process steps: For example, introduce intermediate reasoning path $m$ in $x = (k, r)$, so that the encoding of the first character of $r$ is not generated directly from $k$, but from $(k, m)$.
\end{itemize}

It is easy to observe that the first method corresponds to the pre-training Scaling Laws 
while the second method corresponds to the inference Scaling Laws. 
The second method has multiple implementations, such as chain-of-thought \cite{20}, 
designing complex agent systems based on LLMs \cite{21}, adding a "wait" command to extend the 
model's thinking time \cite{25,26}, and leveraging advanced models like OpenAI's o1 \cite{17} 
and DeepSeek's R1 \cite{27}, among others.
These can all be considered implementations of the second method. 
They expand the value of execution steps by increasing the number of intermediate tokens. 

Finally, we raise a question: Are RNN-type linear structures, 
such as Mamba \cite{22} and RWKV \cite{23}, truly efficient? 
Despite the decoder-only transformer's polynomial-time complexity, its complexity 
does accelerate the approximation of conditional Kolmogorov complexity. 
In contrast, Mamba and RWKV may be less efficient in approximating conditional Kolmogorov complexity.
\vspace{-0.2cm}
\section{Conclusion}
\vspace{-0.2cm}
\label{sec:conclusion}
In this paper, we first demonstrate that the pretraining process of LLMs is a computable 
(total recursive) approximation of the upper bound of joint Kolmogorov complexity. 
Optimizing joint Kolmogorov complexity inherently carries a regularization effect. 
Next, we prove that, under the condition of infinite-precision rational numbers and hardmax attention mechanisms, 
there exists a series of decoder-only transformer such that $\displaystyle \lim_{t \to \infty} M(t, x ,y) = C(x\mid y)$. 
Theoretically, to better approximate $C(x\mid y)$, we need to increase execution steps in $M(t, x, y)$. 
The two types of Scaling Laws represent different methods of increasing execution steps.

\section{Limitations}
\label{sec:limitations}

Our analysis is subject to several notable limitations. First, the existence of 
$y^*$ in \Cref{eq:12} remains unproven, which poses a challenge to the theoretical foundation of our approach.
Second, the current training methodologies for LLMs may not sufficiently approximate the theoretical construct $M(t,x,y)$. 
Lastly, the realization of $M(t,x,y)$ depends on infinite-precision rational numbers and hardmax attention mechanisms.
\clearpage
\bibliography{refer}

\clearpage
\appendix

\section{Example of Arithmetic Coding}
\label{sec:example_of_arithmetic}

Consider a symbol set $\mathcal{S} = \{a,b,c,d,e\}$, and we need to encode the symbol sequence "bab".

Let $p_1 = (\frac{3}{11},\frac{3}{11},\frac{2}{11},\frac{1}{11},\frac{2}{11})$, which divides the $[0,1]$ interval proportionally.

\begin{center}
    \includegraphics[width=0.5\textwidth]{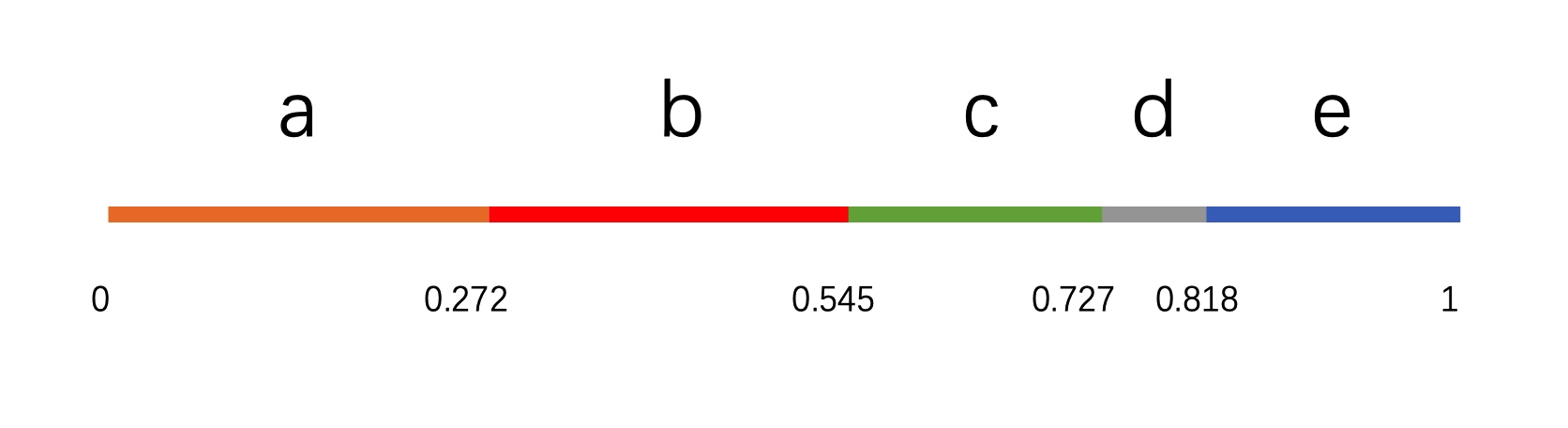}
\end{center}

The first character of the message is $b$, so we select the second segment from the figure above. Assuming $p_2 = (0.1,0.1,0.3,0.3,0.2)$, divide the interval $[0.272,0.545]$ proportionally:

\begin{center}
    \includegraphics[width=0.5\textwidth]{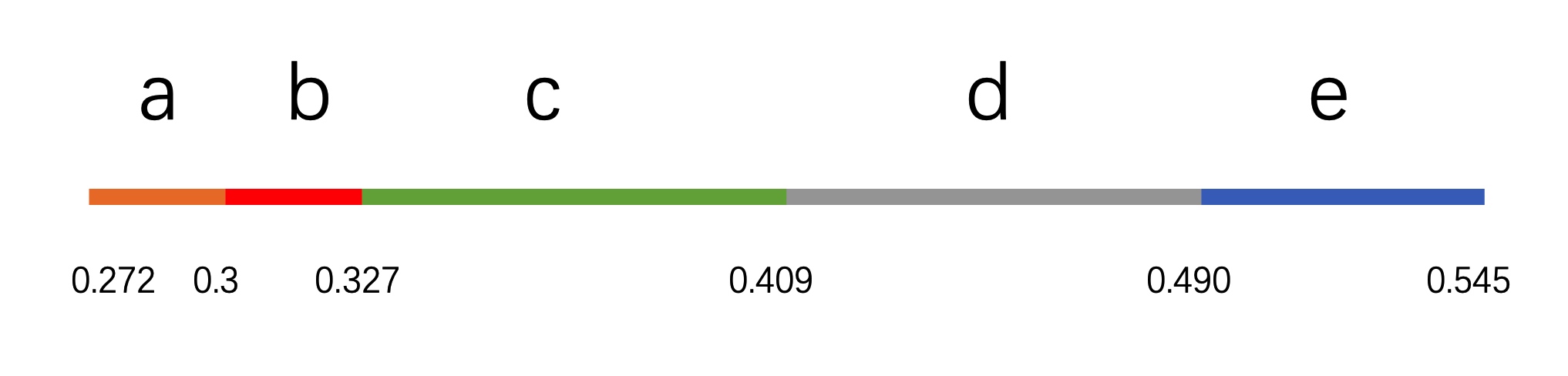}
\end{center}

The second character is $a$, so we select the first segment from the above figure. Assuming $p_3 = (0.2,0.2,0.2,0.2,0.2)$, divide the interval $[0.272,0.3]$ proportionally:

\begin{center}
    \includegraphics[width=0.5\textwidth]{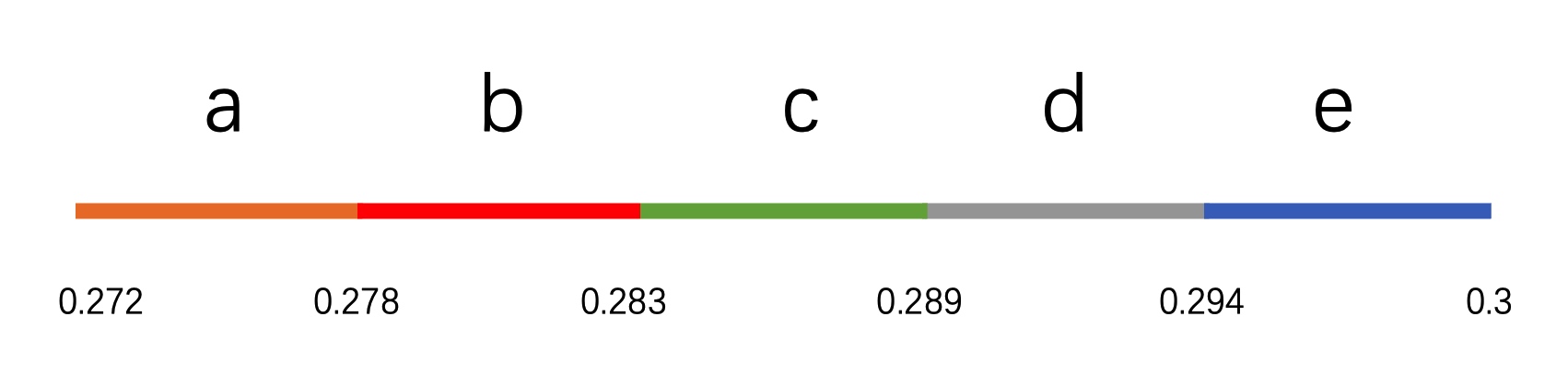}
\end{center}

The last character is $b$, taking out the second segment from above. Within this interval, we find a decimal with the shortest binary representation.
\begin{equation*}
(0.01001)_2 = (0.28125)_{10}
\end{equation*}
The decoding process reverses these operations using the same probability distributions to restore the original symbol sequence. Note that for decoding, besides the arithmetic coding information $(0.01001)_2$, we also need to know $p_1,p_2,p_3$ and the number of decoding iterations.

\section{Proof of Theorem}
\label{sec:proof_of_theorem}

\subsection{Proof of Theorem 4.1}
\label{subsec:proof_of_theorem_41}

\renewcommand{\thetheorem}{4.1}
\begin{theorem}
    Given a universal Turing machine $U$, the joint Kolmogorov complexity $C(x, y)$ of strings $x$ and $y$ satisfies the following inequality:
    \begin{align}
    C(x,y) \leq & C(y) + C(x \mid y)\nonumber \\ 
          & +2l(C(y)) +O(1)
    \end{align} 
    where $l$ represents string length, and $O(1)$ is a constant related to Turing machine $U$.
\end{theorem}

\begin{proof}
    Let $p$ be the shortest program describing $y$, and $q$ be the shortest program describing $x$ given $y$. We can construct the following Turing machine: first use $p$ to describe $y$, then mark the end position of $p$ in the Turing machine with the encoding length $l(p)$, then use program $q$ to describe $x$ based on the previously given $y$. Therefore, exsit a program $\overline{l(p)}pq$ to describe $\langle x,y\rangle$, where $\overline{l(p)}$ represents a prefix code($\overline{x} = 1^{l(x)}0x$). Finally, according to the invariance theorem of Kolmogorov complexity $C(x,y) \leq C(y) + C(x \mid y) + 2l(C(y))+O(1)$.
\end{proof}

\subsection{Proof of Theorem 4.2}
\label{subsec:proof_of_theorem_42}

\renewcommand{\thetheorem}{4.2}
\begin{theorem}
    For each $x$, there exists a total recursive function $\phi(t,x,y)$ that is monotonically decreasing with $t$, such that $\displaystyle \lim_{t \to \infty} \phi(t,x,y)= C(x\mid y)$.
\end{theorem}

\begin{proof}
    Select a universal Turing machine $U$. For each $x$, the length of its shortest program is at most $l(x)+c$. We construct $\phi(t,x,y)$ as follows. On $U$, execute each program $p$ of length at most $l(x)+c$ with input $y$ for $t$ steps. If we can find some $p$ that halts and outputs $x$, then define $\phi(t,x,y)$ as the length of the shortest such $p$; otherwise, let it equal $l(x)+c$. Clearly, $\phi(t,x,y)$ is a total recursive function and monotonically decreases with $t$. Its limit exists because for each $x$, there exists a $t$ such that executing input $p$ and $y$ on $U$ for $t$ steps will output $x$, at which point $l(p) = C(x\mid y)$.
\end{proof}

\end{document}